\documentclass{article}

\PassOptionsToPackage{numbers}{natbib}

\usepackage[preprint]{neurips_2019}

\usepackage[utf8]{inputenc} 
\usepackage[T1]{fontenc}    
\usepackage{hyperref}       
\usepackage{url}            
\usepackage{booktabs}       
\usepackage{amsfonts}       
\usepackage{nicefrac}       
\usepackage{microtype}      

\usepackage{amssymb}
\usepackage{mathtools}
\usepackage{amsthm}
\usepackage{enumitem}
\usepackage{wrapfig}
\usepackage{float}
\usepackage{tikz}
\usepackage{thm-restate}
\usepackage{dsfont}
\usepackage{algorithm}
\usepackage[noend]{algpseudocode}
\usepackage{multirow,array}
\usepackage{wrapfig}
\usepackage{caption}

\newtheorem{theorem}{Theorem}

\newtheorem{definition}{Definition}

\newcommand\encircle[1]{%
	\tikz[baseline=(X.base)] 
	\node (X) [draw, shape=circle, inner sep=0] {\strut #1};}

\DeclareMathOperator*{\argmin}{arg\,min}
\DeclareMathOperator*{\argmax}{arg\,max}

\newcommand{\pl}{\mathcal{P}}
\newcommand{\R}{\mathcal{R}}

\newcommand{\X}{\mathcal{X}}
\newcommand{\I}{\mathcal{I}}

\title{Bayesian Persuasion with Sequential Games}

\author{XYZ}

 \author{%
Andrea Celli\\
Politecnico di Milano\\
\texttt{andrea.celli@polimi.it}\\
\And
Stefano Coniglio\\
University of Southampton\\
\texttt{s.coniglio@soton.ac.uk}\\
\AND
Nicola Gatti\\
Politecnico di Milano\\
\texttt{nicola.gatti@polimi.it}\\
 }

\begin{document}

\maketitle

\begin{abstract}
We study an information-structure design problem (a.k.a.~persuasion) with a \emph{single} sender and \emph{multiple} receivers with actions of \emph{a priori} unknown types, independently drawn from action-specific marginal distributions.
As in the standard \emph{Bayesian persuasion} model, the sender has access to additional information regarding the action types, which she can exploit when committing to a (noisy) signaling scheme through which she sends a private signal to each receiver.
The novelty of our model is in considering the case where the receivers interact in a {\em sequential game with imperfect information}, with utilities depending on the game outcome and the realized action types.
After formalizing the notions of \emph{ex ante} and \emph{ex interim} persuasiveness (which differ in the time at which the receivers commit to following the sender's signaling scheme), we investigate the {\em continuous optimization} problem of computing a signaling scheme which maximizes the sender's expected revenue.
We show that computing an optimal \emph{ex ante} persuasive signaling scheme is \textsf{NP}-hard when there are three or more receivers.
In contrast with previous hardness results for \emph{ex interim} persuasion, we show that, for games with two receivers, an optimal \emph{ex ante} persuasive signaling scheme can be computed in polynomial time thanks to a novel algorithm based on the ellipsoid method which we propose.
\end{abstract}

\section{Introduction}\label{sec:intro}

%
%





\emph{Bayesian persuasion} revolves around influencing the behavior of self-interested agents through the provision of payoff-relevant information~\cite{kamenica2011bayesian}.
Differently from traditional \emph{mechanism design}, where the designer influences the outcome by providing tangible incentives, in Bayesian persuasion the designer influences the outcome by deciding \emph{who gets to know what}~\cite{bergemann2016information}.
Real-world
applications are ubiquitous.
For instance, this framework has been recently applied to security problems~\cite{rabinovich2015information,xu2015exploring,xu2016signaling}, financial-sector stress testing~\cite{goldstein2018stress}, voter coalition formation~\cite{alonso2016persuading}, and online advertisement~\cite{badanidiyuru2018targeting,emek2012}.

The classical Bayesian persuasion model involves a \emph{single sender} and a \emph{single receiver} where the sender, who has access to some private information, designs a signaling scheme in order to persuade the receiver to select a favorable action.
The model assumes the sender's commitment power. 
This hypothesis is realistic in many settings where reputation and credibility are a key factor for the long-term utility of the sender~\cite{rayo2010optimal}, as well as whenever an automated signaling scheme either has to abide by a contractual service agreement or it is enforced by a trusted authority~\cite{dughmi2017algorithmic}.

The extension to the case with multiple receivers is of major interest (see, e.g., its applications to private-value auctions~\cite{kaplan2000strategic}).
In this setting, most of the works assume a \emph{public signal} model in which all the receivers observe the same information~\cite{dughmi2014constrained,alonso2016persuading,dughmi2018hardness}.
A more general setting is the {\em private signal} one, in which the sender may tailor receiver-specific signals.
Persuasion with private signals has been explored only in very specific settings, such as two-agents two-action games~\cite{taneva2018}, unanimity elections~\cite{bardhi2018modes}, binary auctions with no inter-agent externalities~\cite{arieli2016private,babichenko2016computational}, and voting with binary action spaces and binary states of Nature~\cite{wang2013bayesian}.
As pointed out by~\citet{dughmi2017algorithmic}, the problem of computing private signaling schemes in multi-receivers settings still lacks a general algorithmic framework.

Differently from classical Bayesian persuasion models, which typically assume that the receivers take their actions simultaneously~\cite{dughmi2017algorithmic,kamenica2018bayesian}, we address, for the first time in the literature (to the best of our knowledge), the multi-receiver case with sequential interactions among receivers.
As most of the real-world economic interactions take place sequentially, this allows for a greater modeling flexibility which could be exploited in the context of, e.g, sequential auctions~\cite{leme2012sequential}.
In the paper, we show how to address sequential, multi-receiver settings algorithmically via the notion of \emph{ex ante} persuasive signaling scheme, where receivers commit to following the sender's recommendations by only observing the signaling scheme.
This is motivated by the fact that the classical notion of persuasiveness (\emph{ex interim} persuasiveness) which allows the receivers to deviate after observing the sender's signal renders most of the associated design problems (with the exception of very narrow settings) computationally intractable~\cite{dughmi2016algorithmic}, ultimately making its adoption impractical in real-world applications where the receivers act sequentially.

\emph{Ex ante} persuasive signaling schemes may be employed every time the environment allows for a credible receivers' commitment before the recommendations are revealed. 
As argued by~\citet{kamenica2011bayesian}, this is not unrealistic. 
On a general level, the receivers will
uphold their \emph{ex ante} commitment every time they reason with a long-term horizon where a reputation for credibility positively affects their utility~\cite{rayo2010optimal}. 
In some cases, they could also be forced to stick to their \emph{ex ante} commitment by contractual agreements. 
Many real-world problems involve \emph{ex ante} commitments. 
This is the case, for example, of sequential auction in online advertising, where a (trusted) third party service (e.g., programmatic advertising platforms) could allow bidders for coordinated behaviors during the sequential auction, leading to better outcomes in terms of bidders' payoffs, and to more efficient allocations of the ads.
%

\textbf{Original contributions}.
We investigate persuasion games with multiple receivers interacting in a sequential game, and study the continuous optimization problem of computing a private signaling scheme which maximizes the sender's expected utility.
We focus on the framework with independent action types, similarly to what is done by~\citet{dughmi2016algorithmic}.
We introduce the notion of \emph{ex ante} persuasive signaling scheme, and formalize its differences from \emph{ex interim} persuasive schemes.
%
%
Motivated by the hardness results for the \emph{ex interim} setting with simultaneous moves provided by~\citet{dughmi2016algorithmic}, we study the problem of computing optimal \emph{ex ante} signaling schemes.
We show that one such scheme may be computed in polynomial time in settings with two receivers and independent action types, which makes \emph{ex ante} persuasive signaling schemes a plausible
persuasion tool in practice. 
In proving this result, we show that, given any behavioral strategy of a perfect-recall player, it is possible to find, in polynomial time, a realization-equivalent mixed strategy with a polynomially-sized support.
%
when there are three or more receivers.
Moreover, we show that the case with two receivers is the largest one, in terms of players, in which the problem of computing an optimal \emph{ex ante} signaling scheme is tractable by showing that with three or more receivers the problem is~\textsf{NP}-hard.

\vspace{-.3cm}
\section{Bayesian Persuasion with Sequential Games}\label{sec:model}
We assume, in our model, a \emph{sender} 
denoted by $S$ and a set of \emph{receivers} $\R=\{1,\ldots,n\}$.
Each receiver $i\in \R$ is faced with the problem of selecting actions from a set $A_i$ with \emph{a priori} uncertain payoffs.
We adopt the perspective of the sender, whose goal is persuading the receivers to take actions which are favorable for her.
The fundamental feature of our model is that receivers confront themselves in a \emph{sequential} decision problem, which we describe as an \emph{extensive-form game} (EFG) with imperfect information and perfect recall. 

Payoffs are a function of the actions taken by the receivers and of an unknown \emph{state of nature} $\theta$, drawn from a set of potential realizations $\Theta$.
We follow the standard framework of~\citet{dughmi2016algorithmic} where each action $a$ has a set of possible types $\Theta_a$ and in which a state of nature $\theta$ is a vector specifying the realized type of each action of the receivers, i.e., $\theta\in\Theta=\bigtimes_{i\in\R}\bigtimes_{a\in A_i}\Theta_a$.\footnote{Standard (i.e., non Bayesian) EFGs can be represented by assigning to each $\Theta_a$ a singleton. 
Note that this model also encompasses Bayesian games \emph{\'a la}~\citet{harsanyi1967games}.
}
We assume action types which are drawn independently from action-specific marginal distributions.
We denote them by $\tilde\pi_a\in\textnormal{int}(\Delta^{|\Theta_a|})$, where $\tilde\pi_{a}(t)$ is the probability of $a$ having type $t\in\Theta_a$.\footnote{
	$\textnormal{int}(X)$ is the interior of set $X$, and $\Delta^{|X|}$ is the set of all probability distributions on $X$.
}
These marginal distributions form a common prior over the states of nature which we assume to be known explicitly to both sender and receivers.
This common knowledge can be equivalently represented by the distribution $\mu_0\in\Delta^{|\Theta|}$, where $\mu_0(\theta)=\prod_{i\in\R}\prod_{a\in A_i}\tilde\pi_a(\theta_a)$.

We now provide some background on EFGs, and describe our two models of \emph{optimal signaling}.

\subsection{Background on EFGs}\label{subsec:efg}
An EFG---here denoted by $\Gamma$---is composed of a set $H$ of nodes, each of which identified by the ordered sequence of actions leading to it from the root node.
The set of terminal nodes of the game is denoted by $Z\subseteq H$.
The game is played by the receivers $\R$.
$A_i$ is the set of actions available to each receiver $i\in\R$.
Let $A=\{A_i\}_{i\in\R}$.
For each nonterminal node $h\in H\setminus Z$, we denote by, respectively, $P(h)$ and $A(h)$ the unique receiver acting at $h$ and the set of actions available at that node.
Imperfect information is represented via \emph{information sets} (or {\em infosets}), which group together decision nodes which are indistinguishable for a certain receiver.
For each receiver $i$, we denote her set of infosets by $\I_i$.
$\I_i$ defines a partition of $\{ h \in H \mid P(h)=i \}$.
Each $I\in\I_i$ is such that $A(h)=A(h')$ $\forall h,h'\in I$.
To simplify the notation, let $A(I)$ be the set of actions available at each decision node in $I$.
Receiver $i$ has \emph{perfect recall} if she has perfect memory of her past actions and observations.

%
%
We denote a \emph{behavioral strategy} of receiver $i$ by $\pi_i$.
It corresponds to a vector defining a probability distribution over $A(I)$, $\forall I\in\I_i$.
Given $\pi_i$, let $\pi_{i,I}$ be the (sub)vector representing the probability distribution at $I\in\I_i$. 
%
%
Letting, for each receiver $i$, $\Sigma_i=\bigtimes_{I\in\mathcal{I}_i} A(I)$, a \textit{plan} is a vector $\sigma_i\in\Sigma_i $ which specifies an action for each of the receiver's infosets.
We denote by $\sigma_i(I)$ the action selected at infoset $I\in \mathcal{I}_i$. 
Letting $\Sigma=\bigtimes_{i\in \pl}\Sigma_i$, we denote by $\sigma\in\Sigma$ the tuple which specifies the plan chosen by each receiver.
%
%
%
%
Finally, a \textit{mixed strategy} $x_i$ is a probability distribution over $\Sigma_i$. 
We let $\mathcal{X}_i$ be the mixed strategy space of receiver $i$, and $\X$ be the set of joint probability distributions over $\Sigma$.

The {\em sequence form}~\cite{koller1996,von1996} of a game is a compact representation applicable to games with perfect recall.
It decomposes strategies into sequences of actions and their realization probabilities.
A sequence $q_i$ for receiver $i$ associated with a node $h$ is a tuple specifying receiver $i$'s actions on the path from the root to $h$.
We denote the set of all sequences for receiver $i$ by $Q_i$.
A sequence is said terminal if, together with some sequences of the other receivers, leads to a terminal node. 
We let $q_\emptyset$ be the fictitious sequence leading to the root node and $qa$ the extended sequence obtained by appending action $a$ to $q$. 
A \emph{sequence-form strategy} (or \emph{realization plan}) for a receiver $i$ is a function $r_i:Q_i\rightarrow [0,1]$ such that $r_i(q_\emptyset)=1$ and, for each $I\in\mathcal{I}_i$ and sequence $q$ leading to $I$, $-r_i(q)+\sum_{a\in A(I)}r_i(qa)=0$.
%
%
We denote by $Q(I)$ the set of sequences originating in $I$. 
For each $q\in Q_i$, we denote by $I_\downarrow(q)\subseteq \I_i$ the set of infosets reachable by $i$ after selecting $q$ without making other intermediate moves, whereas $I_\uparrow(q)\in\I_i$ denotes the unique infoset where the last action of $q$ was taken.
\!\footnote{When the context requires disambiguation between different games, we write $I_\downarrow^\Gamma(q)$ to denote the result for EFG $\Gamma$.}
%
%
We call two strategies of receiver $i$ \emph{realization equivalent} if, for any fixed strategy of the other receivers, they induce the same distribution over $Z$.


\subsection{\emph{Ex interim} Persuasiveness}

Let $u_S:\Sigma\times\Theta\to\mathbb{R}$ and $u_i:\Sigma\times\Theta\to\mathbb{R}$ be the payoff functions of the sender and receiver $i \in \R$.
We assume that the sender is allowed to tailor signals to individual receivers through private communications.
Let $\Omega_i$ be the set of signals available to receiver $i$, and let $\Omega=\bigtimes_{i\in\R}\Omega_i$. 
We assume that the sender has access to private information and her goal is designing a \emph{signaling scheme} $\varphi: \Theta\to\Delta^{|\Omega|}$ to persuade the receivers to select actions which are favorable for her.
We denote by $\varphi_\theta$ the probability distribution over $\Omega$ having observed $\theta$. 
%
%
In the classical Bayesian persuasion framework~\cite{kamenica2011bayesian}, the receivers decide their behavior after observing the sender's signal and updating their posterior over $\Theta$ accordingly.
The sender-receivers interaction goes as follows:
\begin{itemize}[nolistsep,itemsep=0mm,leftmargin=1cm]
	\item The sender chooses $\varphi$ and publicly discloses it.
	
	\item Nature draws a state $\theta\sim\mu_0$, observed by the sender.
	
	\item The sender draws a tuple $\omega\sim\varphi_\theta$ and privately sends signal $\omega_i$ to each receiver $i\in\R$.
	
	\item Each receiver $i$ updates her posterior distribution knowing $\varphi$ and having observed $\omega_i$. Then, each of the receivers selects a plan $\sigma_i\in\Sigma_i$. Together, their joint choices form the tuple $\sigma=(\sigma_1,\ldots,\sigma_n)$.
	
	\item Sender and receivers receive, respectively, payoffs $u_S(\theta,\sigma)$ and $u_i(\theta,\sigma)$, for all $i\in\R$.
\end{itemize}
In this setting, a result similar to the {\em revelation principle} (see, e.g.,~\cite{myerson79}) holds. 
Specifically, an \emph{optimal signaling scheme} (i.e., a signaling scheme maximizing the sender's expected utility) can always be obtained by restricting the set of signals $\Omega$ to the set of plans $\Sigma$ (see~\cite[Proposition 1]{kamenica2011bayesian}).
In the following, we assume $\Omega=\Sigma$ (i.e., the sender recommends a plan to follow to each receiver).
The receivers have an incentive to follow the sender's recommendation $\hat\sigma_i$ if the recommended plan is preferred to any other action, conditional on the knowledge of $\hat\sigma_i$. 
We call this condition \emph{ex interim persuasiveness}, which is precisely the kind of constraint characterizing a \emph{Bayes correlated equilibrium} (BCE)~\cite{bergemann2013robust,bergemann2016bayes}. We remark that, according to the definition of BCE, the signaling scheme must necessarily be defined on plans and cannot be compactly represented by using sequences or actions.
\begin{definition}[\emph{Ex interim} persuasiveness]\label{def:ex_int}
	 A signaling scheme $\varphi:\Theta\to\Delta^{|\Sigma|}$ is \emph{ex interim persuasive} if the following holds for all $i\in\R$ and $\sigma_i,\sigma_i'\in\Sigma_i$:
	 \[\hspace{-.3cm}
		 \sum_{\substack{\theta\in\Theta,\\\sigma_{-i}\in \Sigma_{-i}}}\hspace{-.3cm}\mu_0(\theta)\varphi_\theta(\sigma_i,\sigma_{-i})\Big(u_i(\theta,(\sigma_i,\sigma_{-i}))-u_i(\theta,(\sigma'_i,\sigma_{-i}))\Big)\geq 0.
	 \]
\end{definition}
\begin{definition}\label{def:bce}
	A signaling scheme $\varphi:\Theta\to\Delta^{|\Sigma|}$ is a BCE if it is \emph{ex interim} persuasive.
\end{definition}

\subsection{\emph{Ex ante} Persuasiveness}
We introduce the setting in which receivers have to decide whether to follow the sender's recommendations before actually observing them, basing their decision only on the knowledge of $\mu_0$ and $\varphi$.
\!%
\footnote{The receivers' commitment to follow a certain signaling scheme is not an unrealistic assumption for the same reason why it is realistic to assume the sender's commitment power  (see Section~\ref{sec:intro}).}
The interaction between sender and receivers goes as follows:
\begin{itemize}[nolistsep,itemsep=0mm,leftmargin=1cm]
	\item The sender computes $\varphi$, and publicly discloses it.
	
	\item The receivers decide whether to adhere to the recommendations drawn according to $\varphi$ or not.
	
	\item Nature draws a state $\theta\sim\mu_0$, observed by the sender.
	
	\item If $i\in\R$ decided to opt-in to the signaling scheme:
			\begin{itemize}
				\item[$\circ$] the sender draws $\hat \sigma_i\sim\varphi_\theta$ and privately communicates it to receiver $i$;
				
				\item[$\circ$] receiver $i$ acts according to the recommended $\hat \sigma_i$.
			\end{itemize}
	
	\item Sender and receivers receive, respectively, payoffs $u_S(\theta,\sigma)$ and $u_i(\theta,\sigma)$, $\forall i\in\R$, where 
	$\sigma_i=\hat\sigma_i$ if $i$ adhered to the scheme.
\end{itemize}
In this setting, the receivers adhere to the signaling scheme (i.e., $\sigma_i=\hat\sigma_i$) if it is \emph{ex ante} persuasive:
\begin{definition}[\emph{Ex ante} persuasiveness]\label{def:ex_ante}
	The signaling scheme $\varphi:\Theta\to\Delta^{|\Sigma|}$ is \emph{ex ante} persuasive if, for all $i\in\R$ and $\sigma_i\in\Sigma_i$, the following holds:
	%
	$$
	\hspace{-0.3cm}\sum_{\substack{\theta\in\Theta,\sigma_i'\in\Sigma_i\\\sigma_{-i}\in \Sigma_{-i}}} \hspace{-.5cm}\mu_0(\theta)\varphi_\theta(\sigma_i',\sigma_{-i})\Big(u_i(\theta,(\sigma'_i,\sigma_{-i}))-u_i(\theta,(\sigma_i,\sigma_{-i}))\Big)\geq 0.
	$$
\end{definition}
\noindent
Such constraints characterize \emph{Bayes coarse correlated equilibria} (BCCE), i.e., the generalization of coarse correlated equilibria to incomplete-information games (see~\cite{Forges1993,hartline2015no,caragiannis2015bounding}):
\!%
\footnote{The set of (non Bayesian) coarse correlated equilibria is characterized by the constraints of Definition~\ref{def:ex_ante}, with $|\Theta_a|=1$ $\forall a\in A$.}
\begin{definition}\label{def:bcce}
	A signaling scheme $\varphi:\Theta\to\Delta^{|\Sigma|}$ is a BCCE if it is \emph{ex ante} persuasive.
\end{definition}

\subsection{Comparison}

Figure~\ref{fig:line_example} summarizes the interaction flow between sender and receivers in the two settings described above.
The key difference is the time at which the receivers decide whether to adhere to the signaling scheme or not.
%
%
%
\begin{figure}[h!]
	\begin{minipage}{.6\textwidth}
		\includegraphics[scale=1.1]{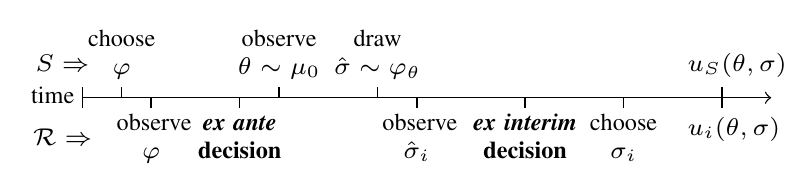}
	\end{minipage}
	\begin{minipage}{.4\textwidth}
		\vspace{-.3cm}
		\scriptsize
		\setlength{\extrarowheight}{3pt}
		\begin{tabular}{cc|c|c|c|}
			& \multicolumn{1}{c}{} & \multicolumn{1}{c}{$In$}  & \multicolumn{1}{c}{$Out$} & \multicolumn{1}{c}{$P$} \\\cline{3-5}
			\multirow{2}*{}  & $E$ & $(-1,1)$ & $(1,0)$ & $(0,1/2)$ \\\cline{3-5}
			& $H$ & $(-1,-1)$ & $(1,0)$ & $(0,0)$ \\\cline{3-5}
		\end{tabular}
	\end{minipage}
	\caption{\emph{Left:} Interaction between sender and receivers in the \emph{ex ante} and \emph{ex interim} setting. \emph{Right:} A game where \emph{ex ante} persuasion guarantees the sender a higher expected utility with respect to \emph{ex interim} persuasiveness.}
	\label{fig:line_example}
\end{figure}
We also propose the following illustrative example (in the basic single-receiver setting) to further illustrate the fundamental differences between the two notions of persuasiveness.


\textbf{Example 1.}
{\em The incumbent of an industry wants to persuade a potential new entrant to the market. 
The market can be either easy ($E$), with probability 0.3, or hard ($H$). 
The incumbent knows the state of the market. 
The entrant has three possible actions: entering the market ($In$), staying out of the market ($Out$), or proposing a partnership to the incumbent ($P$).  Figure~\ref{fig:line_example} depicts the utility matrix for the game (the first values are incumbent’s payoffs).}

{\em The incumbent wants the entrant to stay out of the market, values its entrance negatively, and is indifferent towards a partnership. 
The entrant values entering the new market positively only when it has favorable conditions. 
A partnership in a hard market gives the entrant 0 (rather than a negative score) as no fixed costs have to be sustained.
In this setting, forcing the entrant (contractually) to commit to following the incumbent’s recommendations \emph{ex ante} is strictly better (in terms of expected utility) for the incumbent.} 

{\em An optimal \emph{ex ante} signaling scheme (e.g., $\varphi_E(In)=\varphi_E(Out)=\frac{1}{2}$, $\varphi_H(Out)=1$) guarantees the sender an expected utility of $0.7$.
An optimal \emph{ex interim} signaling scheme (e.g., $\varphi_E(P)=1$, $\varphi_H(Out)=\frac{11}{14}$, $\varphi_H(P)=\frac{3}{14}$,) guarantees a sender's expected utility of 0.55.
Therefore, \emph{ex ante} persuasion provides a 27\% increase in utility for the incumbent w.r.t. ex interim persuasion.}

%
We remark that the set of \emph{ex ante} persuasive signaling schemes strictly includes the set of \emph{ex interim} signaling schemes.
In particular,
an optimal \emph{ex ante} persuasive signaling scheme may lead to an expected utility for the sender arbitrarily larger than the one she would obtain with an optimal \emph{ex interim} scheme.
%
This is shown via the following example:
%

%
%
%

\begin{wrapfigure}{R}{.5\textwidth}
	\centering
	\includegraphics[scale=1.0]{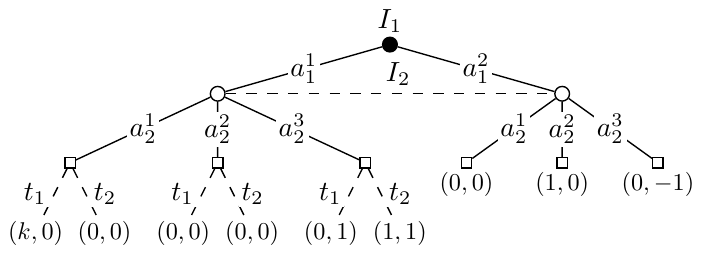}
	\caption{A game with two receivers in which action $a_1^1$ has two possible types $t_1$ and $t_2$. 
		Terminal nodes report receivers' utilities.}
	\label{fig:example}
\end{wrapfigure}
\textbf{Example 2.}
{\em Consider the game in Figure~\ref{fig:example}, with two receivers with one information set each ($I_1$ for receiver 1 and $I_2$ for receiver 2), and parametric in $k\gg 1$.
Action $a_1^1\in A_1$ is such that $\Theta_{a_1^1}=\{t_1,t_2\}$ and $\tilde\pi_{a_1^1}(t_1)=\tilde\pi_{a_1^1}(t_2)=1/2$.
The figure only reports the receivers' utilities, as we assume $u_S(\theta,\sigma)=u_1(\theta,\sigma)+u_2(\theta,\sigma)$, $\forall (\theta,\sigma)$.}
{\em 	The signaling scheme with $\varphi_{t_1}'(a_1^1,a_2^1)=1/2$, $\varphi_{t_1}'(a_1^2,a_2^2)=1/2$, and $\varphi_{t_2}'(a_1^1,a_2^3)=1$ is \emph{ex ante persuasive} but not \emph{ex interim} persuasive.
	The optimal \emph{ex interim} persuasive signaling scheme is such that $\varphi_{t_1}''(a_1^2,a_2^2)=1$, and $\varphi_{t_2}''(a_1^1,a_2^3)=1$. 
	Scheme $\varphi'$ grants the sender an expected utility of $(k+5)/4$, while scheme $\varphi''$ guarantees $3/2$.
	Therefore, for increasing values of $k$ an optimal \emph{ex ante} signaling scheme provides an arbitrarily larger utility than what can be obtained by \emph{ex interim} persuasion.}

\section{Positive Result}
In the independent-actions setting, it is known that computing an optimal \emph{ex interim} signaling scheme is \#\textsf{P}-hard even with a single receiver~\cite{dughmi2016algorithmic}.
Motivated by this negative result, we study the problem of computing an optimal (for the sender) \emph{ex ante} persuasive signaling scheme.
We denote this problem by \textsf{OPT-EA}.
It amounts to computing a coarse correlated equilibrium (CCE) for the game of complete information obtained by treating Nature as a player with a \emph{trivial} (i.e., constant everywhere) payoff function and subject to having marginal strategies constrained to be equal to $\mu_0$.

%
In contrast with the known hardness results for the \emph{ex interim} setting, we show that \textsf{OPT-EA} with $|\R|=2$ can be solved in polynomial time (see Theorem~\ref{th:positive}).
To prove our main theorem, we first show how to build, in polynomial time, \emph{small} (i.e., with a support of polynomial size) mixed strategies which are realization-equivalent to a given behavioral strategy.
Omitted proofs are presented in Appendix~\ref{sec:omitted}.


\subsection{Small Supported Mixed Strategies}\label{subsec:strategies}
Given a behavioral strategy profile $\pi^\ast_i$ for a generic perfect-recall player $i$, we show (see Theorem~\ref{th:poly_ricostruzione}) that it is always possible to compute in polynomial time some $x_i^\ast\in\X_i$ such that (i) it is realization-equivalent to $\pi^\ast_i$ and (ii) it has a support of polynomial size.\footnote{As customary, we define the support of a mixed strategy $x_i\in\X_i$ as $\textnormal{supp}(x_i)\coloneqq\{\sigma_i\in\Sigma_i|x_i(\sigma_i)>0\}$.}

For each $\sigma_i\in \Sigma_i$, let $\xi(\sigma_i)\coloneqq\{q\in Q_i|\exists I\in\I_i,\sigma_i(I)=q\}$ (i.e., the set of sequences selected with probability 1 in a realization plan equivalent to $\sigma_i$).
Analogously, $\forall\sigma=(\sigma_1,\sigma_2)\in\Sigma$ we denote by $\xi(\sigma)$ the set of tuples $(q_1,q_2)$ such that $q_1\in\xi(\sigma_1)$ and $q_2\in\xi(\sigma_2)$.
In the remainder of the section, we drop the dependency on $i$ when not strictly necessary.
We denote by $M$ an $|Q_i|\times|\Sigma_i|$ matrix where $M(q,\sigma)=1$ iff $q\in\xi(\sigma)$ and $M(q,\sigma)=0$ otherwise.
We denote by $M_{q}$ the row of $M$ specifying the plans containing $q$ in their support.
Let $r^\ast$ be the $|Q_i|$-dimensional vector representing the realization plan of player $i$ which is realization-equivalent to $\pi^\ast$.
In order to compute $x^\ast$, it is enough to find an optimal solution to LP
%
$
\max_{x\in\mathbb{R}_{\geq 0}^{|\Sigma_i|}}\quad \{\mathds{1}^{\top}x
\hspace{.4cm}\textnormal{s.t.}\quad  Mx\leq r^\ast \},
$
which we denote by \encircle{\textsf{A}}, which has a polynomial number of constraints and an exponential number of variables.

By relying on the assumption of perfect recall and proceeding by contradiction, we establish the following lemma:
\begin{restatable}{lemma}{lpRicostruzione}\label{lemma:lp_ricostruzione}
	An optimal solution $x^\ast$ to \encircle{$\mathsf{A}$} satisfies $M x^\ast=r^\ast$.
\end{restatable}
%

We now characterize an optimal solution to \encircle{$\mathsf{A}$} by  two properties which are proven by using Lemma~\ref{lemma:lp_ricostruzione} and the fact that, as LP \encircle{$\mathsf{A}$} contains a polynomial number of constraints, it admits an optimal basic solution with only a polynomial number of nonzero variables:
\begin{restatable}{theorem}{thRicostruzione}\label{th:ricostruzione}
	These two properties hold:
	(i) An optimal solution $x^\ast$ to \encircle{$\mathsf{A}$} is a normal-form strategy ($x^\ast\in\X_i$) realization equivalent to $r^\ast$,
	(ii) there exists an optimal solution $x^\ast$ with $\textnormal{supp}(x^\ast)$ of polynomial size.
      \end{restatable}
      %
%

Let $\mathcal{D}$ be the dual of problem \encircle{$\mathsf{A}$}.
By showing that an optimal plan corresponding to a violated dual constraint can be found in polynomial time by backward induction, we establish the following:
\begin{restatable}{lemma}{lemmaDual}\label{lemma:dual_ricostruzione}
	$\mathcal{D}$ admits a polynomial-time separation oracle.
\end{restatable}
Next, by relying on Lemma~\ref{lemma:dual_ricostruzione} and on the ellipsoid method we prove a result which is the basis for our main theorem, Theorem~\ref{th:positive} (whose statement and proof are given in full in the next subsection):
\begin{restatable}{theorem}{thPolyRec}\label{th:poly_ricostruzione}
	Given an EFG, a perfect-recall player $i$, and a behavioral strategy profile $\pi^\ast$ for $i$ (with the realization-equivalent realization plan $r^\ast$), a solution to LP \encircle{$\mathsf{A}$} can be found in polynomial time.
\end{restatable}
%
Finally, we show that we can efficiently compute a solution with support size of at most $|Q_i|$ by applying the ellipsoid method for at most a polynomial number of iterations:
\begin{restatable}{corollary}{corPolySol}\label{corollary:poly_sol}
		A \emph{basic feasible solution} to \encircle{$\mathsf{A}$} can be computed in polynomial time.
\end{restatable}

\subsection{Optimal \emph{Ex Ante} Persuasive Schemes}\label{subsec:bcce_positive}
Computing an \emph{ex ante} persuasive signaling scheme is equivalent to computing a CCE for an EFG of complete information where Nature is treated as a player with constant utility and marginal strategies constrained to be equal to $\mu_0$.
We focus on the setting where $|\R|=2$ and show that \textsf{OPT-EA} can be solved in polynomial time.
%
%
We reason over an auxiliary game such that each action of the receivers is followed by one of Nature's nodes, determining its type.
%
%
Marginal probabilities $\tilde\pi$ determining action types are treated as behavioral strategies of the Nature player, which we denote by $\mathsf{N}$. 
Formally:

\begin{definition}\label{def:auxgame}
	Given an EFG $\Gamma$ describing the interaction between receivers and a set of marginal distributions $\{\tilde\pi_a\in\textnormal{int}(\Delta^{|\Theta_a|})\}_{a\in A}$, the auxiliary game $\hat \Gamma$ is an EFG such that:
	\begin{itemize}[nolistsep,itemsep=0mm,leftmargin=1cm]
		\item It has a set of players $\R\cup\{\mathsf{N}\}$.
		
		\item For each receiver $i\in\R$, her utility function is the same as in $\Gamma$, i.e., $\forall (\theta,\sigma)\in\Theta\times\Sigma$, $u_i(\theta,\sigma)=\hat u_i(\theta,\sigma)$. 
		Nature has $\hat u_\mathsf{N}(\cdot)=k\in\mathbb{R}$ constant everywhere.
		
		\item The receivers have the same information structures as in $\Gamma$, i.e.,  $\forall i\in\R$, $\I_i=\hat \I_i$, and $\forall q\in Q_i$, $I_\downarrow^\Gamma(q)=I_\downarrow^{\hat\Gamma}(q)$.
		
		\item $\forall i\in\R$, each $a\in A_i$ is immediately followed by a singleton infoset $I\in\I_\mathsf{N}$ such that $A(I)=\Theta_a$.
		
		\item $\forall I\in\I_\mathsf{N}$, with $I$ following $a\in A$, $\mathsf{N}$ selects actions (types) at $I$ according to the marginal distribution $\tilde\pi_a$.
	\end{itemize}
\end{definition}

The first step is devising an LP to compute a BCCE with a polynomial number of constraints and an exponential number of variables.
We do so by providing an LP to find an optimal CCE over $\hat \Gamma$.
%
%
First, notice that $\theta$ is a plan of player $\mathsf{N}$ in $\hat\Gamma$.
A distribution in $\Delta^{|\Theta|}$ is a mixed strategy of $\mathsf{N}$. 
Denote by $\mu^\ast$ the mixed strategy realization equivalent to $\tilde\pi$ computed (in poly-time) as in the proof of Theorem~\ref{th:poly_ricostruzione}.
Let $\Theta^\ast\coloneqq\textnormal{supp}(\mu^\ast)$.
Due to Corollary~\ref{corollary:poly_sol}, the set $\Theta^\ast$ has polynomial size.
%
%
Then, we write the problem as a function of $\gamma\in\Delta^{|\Sigma\times\Theta^\ast|}$ (i.e., we look for a correlated distribution in $\hat\Gamma$, encompassing the Nature player).
Let $v_i$ be the $|\I_i|$-dimensional vector of variables of the dual of the best-response problem for receiver $i$ in sequence form.
Moreover, we employ sparse $(|\R|+1)$-dimensional matrices describing the utility function of sender and receivers for the profiles $(\theta,q_1,q_2)$ leading to terminal nodes of $\hat\Gamma$.
We denote them by $U_i\in\mathbb{R}^{|\Theta^\ast|\times |Q_1|\times |Q_2|}$, with $i \in \R \cup \{S\}$.
\!%
\footnote{$U_i$ employs both the sequence form (for receivers), and plans of \textsf{N}.
	However, polynomiality of $\Theta^\ast$ implies polynomiality of $U_i$.}
In the following, we let $q=(q_1,q_2)$.
The problem of computing a CCE over $\hat\Gamma$ reads:

\hspace{-.4cm}
\begin{minipage}{.99\linewidth}
	\scriptsize
	\begin{align}
	\max_{\substack{\gamma\geq 0,\\v_1,v_2}}&\sum_{\substack{\theta\in\Theta^\ast\\\sigma\in\Sigma}}\gamma(\theta,\sigma)\sum_{q\in\xi(\sigma)}U_S(\theta,q)\label{eq:cce2_fo}&\\
	\textnormal{s.t.}\quad& \sum_{\substack{\theta\in\Theta^\ast\\\sigma\in\Sigma}}\gamma(\theta,\sigma)\sum_{q\in\xi(\sigma)}U_i(\theta,q)\geq \hspace{-.3cm}\sum_{\substack{I'\in\I_i:\\I'\in I_\downarrow(q_{\emptyset})}}\hspace{-.3cm}v_i(I') \hspace{.45cm}\forall i\in\R \label{eq:cce2_util}&\\
	& v_1(I_\uparrow(q_1))-\hspace{-.3cm}\sum_{I'\in I_\downarrow(q_1)}\hspace{-.3cm}v_1(I')-\hspace{-.1cm}\sum_{\substack{\theta\in\Theta^\ast\\\hspace{-.15cm}\sigma\in\Sigma}}\hspace{-.1cm}\gamma(\theta,\sigma)\hspace{-.3cm}\sum_{q_2\in\xi(\sigma_2)}\hspace{-.3cm}U_1(\theta,q_1,q_2)\geq 0&\forall q_1\in Q_1&\label{eq:cce2_ic1}\\
	& v_2(I_\uparrow(q_2))-\hspace{-.3cm}\sum_{I'\in I_\downarrow(q_2)}\hspace{-.3cm}v_2(I')-\hspace{-.1cm}\sum_{\substack{\theta\in\Theta^\ast\\\hspace{-.15cm}\sigma\in\Sigma}}\hspace{-.1cm}\gamma(\theta,\sigma)\hspace{-.3cm}\sum_{q_1\in\xi(\sigma_1)}\hspace{-.3cm}U_2(\theta,q_1,q_2)\geq 0&\forall q_2\in Q_2&\label{eq:cce2_ic2}\\
	&\sum_{\sigma\in\Sigma}\gamma(\theta,\sigma)=\mu^\ast(\theta)&\forall\theta\in\Theta^\ast\label{eq:cce2_nature}.
	\end{align}
\end{minipage}


\begin{wrapfigure}{r}{8cm}
	\vspace{.cm}
	\centering
	\includegraphics[width=8cm]{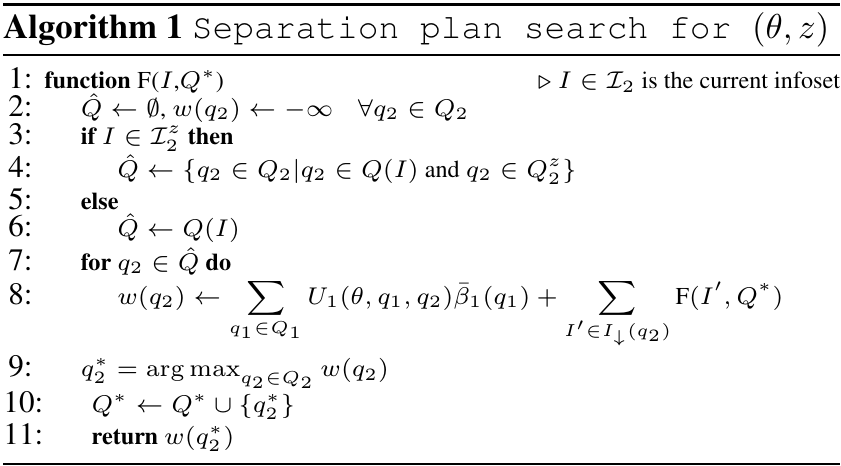}
\end{wrapfigure}

We make the following observations on the above LP, which we denote by \encircle{\textsf{B}}:
\begin{itemize}[nolistsep,itemsep=0mm, leftmargin=.5cm]
	\item The left term of constr.~\eqref{eq:cce2_util} is the expected utility of $i$ at the equilibrium.
	Incentive constraints~\eqref{eq:cce2_ic1} and~\eqref{eq:cce2_ic2} are compactly encoded by exploiting the sequence form. 
	Intuitively, we decompose the best-response problem locally at each infoset.
	The constraints impose that the utility at the equilibrium be no smaller than the value achieved when playing the plan obtained by letting the receiver best respond at each infoset.
\end{itemize}

\begin{itemize}[nolistsep,itemsep=0mm, leftmargin=.5cm]
	\item Constraint~\eqref{eq:cce2_nature} forces Nature's marginal distribution to equal the prior $\mu^\ast$.
	
	\item Once a solution $\gamma^\ast$ to \encircle{\textsf{B}} has been computed, an optimal solution to \textsf{OPT-EA} is the signaling scheme which, having observed $\theta$, recommends $\sigma$ with probability 
	$\gamma^\ast(\theta,\sigma)/\mu^\ast(\theta)$.
\end{itemize}

%

%
The following key positive result holds:
\begin{theorem}\label{th:positive}
	\textsf{OPT-EA} can be solved in polynomial time when $|\R|\leq 2$.
\end{theorem}
\begin{proof}
	%
	%
	Let $\mathcal{D}_B$ be the dual of \encircle{\textsf{B}}.
	Let $\alpha_1$, $\alpha_2$ be the dual variables of constraints~\eqref{eq:cce2_util}, $\beta_1\in\mathbb{R}^{|Q_1|}$ and $\beta_2\in\mathbb{R}^{|Q_2|}$ the dual variables of~\eqref{eq:cce2_ic1} and~\eqref{eq:cce2_ic2}, and $\delta\in\mathbb{R}^{|\Theta^\ast|}$ the dual variables of~\eqref{eq:cce2_nature}.
	We show that, given $(\bar\alpha_1,\bar\alpha_2,\bar\beta_1,\bar\beta_2,\bar\delta)$, the problem of finding either a hyperplane separating the solution from the feasible set of $\mathcal{D}_B$ or proving that no such hyperplane exists can be solved in polynomial time.
	Along the lines of Theorem~\ref{th:poly_ricostruzione}, this implies that \encircle{\textsf{B}} is solvable in polynomial time by the ellipsoid method.
	%
	%
	%
	As the number of dual constraints corresponding to variables $v_i$ is linear, all these constraints can be checked efficiently for violation.
	%
	%
	Besides those, the dual problem $\mathcal{D}_B$ features the following constraint for each $(\theta,\sigma)\in\Theta^\ast\times\Sigma$:
	
	\begin{scriptsize}
	\begin{equation*}
\sum_{i\in\R}\sum_{q\in\xi(\sigma)}U_i(\theta,q)\bar\alpha_i+\frac{\bar\delta(\theta)}{\mu^\ast(\theta)}-\sum_{q\in\xi(\sigma)}\hspace{-.1cm}U_S(\theta,q)-\hspace{-.3cm}\sum_{q\in Q_1\times\xi(\sigma_2)}\hspace{-.2cm}U_1(\theta,q)\bar\beta_1(q_1)
-\hspace{-.3cm}\sum_{q\in\xi(\sigma_1)\times Q_2}\hspace{-.3cm}U_2(\theta,q)\bar\beta_2(q_2)\geq 0.
	\end{equation*}
	\end{scriptsize}

	Given $(\bar\alpha_1,\bar\alpha_2,\bar\beta_1,\bar\beta_2,\bar\delta)$, the \emph{separation problem} of finding a maximally violated inequality of $\mathcal{D}_B$ reads:
	
	\begin{scriptsize}
	\begin{equation*}
\min_{\theta,\sigma}\left\{\sum_{q\in\xi(\sigma)}\hspace{-.1cm}\left[\sum_{i\in\R}U_i(\theta,q)-U_S(\theta,q)\right]+\frac{\bar\delta(\theta)}{\mu^\ast(\theta)}\right.-\hspace{-.1cm}\sum_{q\in Q_1\times\xi(\sigma_2)}\hspace{-.5cm}U_1(\theta,q)\bar\beta_1(q_1)
\left.-\hspace{-.1cm}\sum_{q\in \xi(\sigma_1)\times Q_2}\hspace{-.5cm}U_2(\theta,q)\bar\beta_2(q_2)\right\}.
	\end{equation*}
	\end{scriptsize}
	
	
	A pair $(\theta,\sigma)$ yielding a violated inequality exists iff the separation problem admits an optimal solution of value $<0$.
	If such a $(\theta,\sigma)$ exists, it can be determined in polynomial time by enumerating over the (polynomially many) $(\theta,z)\in\Theta^\ast\times \hat Z$, where $\hat Z$ is the outcomes set of $\hat\Gamma$.
	For each pair $(\theta,z)$, we look for a $\sigma\in\Sigma$ which, together with some actions of \textsf{N}, minimizes the objective function of the separation problem and could lead to $z$.
	The procedure halts as soon as a plan $\sigma$ such that $(\theta,\sigma)$ yielding a violated inequality is found; if it terminates without finding any, $\mathcal{D}_B$ has been solved.
	First, by fixing a pair $(\theta,z)$ the first two terms of the objective function are completely determined.
	The remaining terms can be minimized independently for each receiver. 
	Let us consider the problem of finding $\sigma_2\in\Sigma_2$ (the other one is solved analogously).
	It reads
	$
	\max_{\sigma_2\in\Sigma_2}\{\sum_{q_1\in Q_1}\sum_{q_2\in\xi(\sigma_2)}U_1(\theta,q_1,q_2)\bar\beta_1(q_1)\}$, subject to the constraint that $\sigma_2$ is an admissible plan for the given $z$ (i.e., given the solution plan, it has to be possible to reach $z$ together with some actions of the other players).
	%
	%
	This problem can be solved in poly-time as shown in Algorithm 1, where $\I_i^z$ and $Q_i^z$ are, respectively, the set of infosets and sequences of $i$ encountered on the path from the root to~$z$.
	%
%
	Once $Q^\ast$ has been determined by visiting each $I\in\I_2$, the corresponding optimal $\sigma_2$ can be built directly.
	%
	%
	As in Corollary~\ref{corollary:poly_sol}, an optimal solution to \encircle{\textsf{B}}  has polynomial support size. 
	Then, it is used to determine an optimal solution to \textsf{OPT-EA} in poly-time.
\end{proof}

\section{Negative Result}\label{subsec:bcce_negative}


We conclude by showing that the approach of Section~\ref{subsec:bcce_positive} cannot be extended to settings where $|\R|>2$ and that, in particular, the border between easy and hard cases coincides with $|\R|=2$.
%
%
Indeed, the fact that computing an optimal CCE for a three-player EFG is \textsf{NP}-hard~\cite[Th. 1.3]{vonStengel2008} directly implies the following:
%
%
\begin{restatable}{theorem}{thNpHard}
	\textsf{OPT-EA} is \textsf{NP}-hard when $|\R|>2$.
\end{restatable}
%
In terms of $\mathcal{D}_B$, as a consequence of the equivalence between optimization and separation~\cite{grotschel2012geometric} the previous results maps in the following:
%
\begin{restatable}{theorem}{thDNpHard}
	Computing an optimal solution to $\mathcal{D}_B$ is \textsf{NP}-hard when $|\R|>2$.
\end{restatable}

\section{Discussion}
We have studied persuasion in the multi-receiver setting with private signals, introducing, for the first time, a model encompassing receivers with sequential interactions, as well as the notion of \emph{ex ante} persuasiveness.
In contrast with previous complexity results on computing optimal CCEs and optimal \emph{ex interim} persuasive schemes, we have shown that with $|\R|\leq 2$ an optimal \emph{ex ante} scheme can be computed in polynomial time with the ellipsoid method by relying on a polynomial-time separation oracle.
%
We have also shown that $|\R|=2$ constitutes the border between easy and hard cases as, even for $|\R|=3$, the problem is \textsf{NP}-hard.









\clearpage
\bibliographystyle{plainnat}
\bibliography{ref}

\clearpage
\appendix

\section{Omitted Proofs}\label{sec:omitted}

\lpRicostruzione*
\begin{proof}
	Consider a behavioral strategy $\pi^\ast$ whose realization-equivalent realization plan is denoted by $r^\ast$.
	Since player $i$ has perfect recall, there always exists at least a mixed strategy $\hat x\in\X_i$ realization equivalent to $\pi^\ast$~\cite[Th. 6.11]{maschler_solan_zamir_2013}.
	Therefore, the optimal value of \encircle{$\mathsf{A}$} is 1 (since $\mathds{1}^{\top}\hat x=1$).
	Given $\hat x\in\X_i$, a distribution assigning to each sequence $q\in Q_i$ value $\sum_{\sigma\in\Sigma_i: q\in\xi(\sigma)}\hat x_{\sigma}$ is a valid realization plan.
	Therefore, if $x\in\Delta^{|\Sigma_i|}$, then $Mx$ is a well defined realization plan for $i$.
	Now, by contradiction, assume that $x^\ast$ is an optimal solution of \encircle{$\mathsf{A}$} and that there exists $q'\in Q_i$ such that $M_{q'}x^\ast<r^\ast(q')$.
	Optimality implies $\mathds{1}^\top x^\ast=1$ and, therefore, $x^\ast\in\Delta^{|\Sigma_i|}$. 
	Let $Mx^\ast=r$. We have $r(q')<r^\ast(q')$. Since the sequence-form constraints hold, there must exist at least one $q''\in Q_i$ such that $r(q'')>r^\ast(q'')$. 
	This leads to a contradiction since $x^\ast$ would not be a feasible solution.
\end{proof}

\thRicostruzione*
\begin{proof}
	Since $Mx^\ast=r^\ast$ (Lemma~\ref{lemma:lp_ricostruzione}), we have $M_{q_{\emptyset}}x^\ast=r^\ast(q_\emptyset)$, that is $\mathds{1}^\top x^\ast=1$. Therefore, $x^\ast\in\X_i$.
	Realization equivalence follows from Lemma~\ref{lemma:lp_ricostruzione} and from the fact that $Mx^\ast$ defines a valid realization plan. 
	Moreover, LP \encircle{$\mathsf{A}$} admits a basic optimal solution with at most $|Q_i|$ variables with strictly positive values~\cite{shapley1950basic}. 
	Then, there exists an optimal $x^\ast$ with support of polynomial size. 
\end{proof} 

\lemmaDual*
\begin{proof}
	Let $\alpha\in\mathbb{R}^{|Q_i|}$ be the vector of dual variables (corresponding to constraints $Mx\leq r^\ast$). 
	$\mathcal{D}$ is an LP with a polynomial number of variables ($|Q_i|$) and an exponential number of constraints ($|\Sigma_i|$).
	We show that, given $\bar \alpha\in\mathbb{R}^{|Q_i|}$, the problem of finding a hyperplane separating $\bar\alpha$ from the set of feasible solutions to $\mathcal{D}$ or proving that no such hyperplane exists can be solved in polynomial time. 
	The problem amounts to determining whether there exists a violated (dual) constraint $M^\top_\sigma \bar\alpha\geq1$ for some $\sigma\in\Sigma_i$.
	Given $\bar\alpha$, the \emph{separation problem} of finding one such constraint of maximum violation reads: $\min_{\sigma\in\Sigma_i}\{M_{\sigma}^\top\bar\alpha\}$.
	A plan $\sigma$ yielding a violated constraint exists iff the separation problem admits an optimal solution of value $<1$.	
	One such plan (if any) can be found in polynomial time by reasoning in a backward induction fashion, starting from information sets of $i$ originating only terminal sequences, and proceeding backwards. At each $I\in\I_i$, player $i$ selects $\hat q\in Q(I)$ such that
	$
	\hat q \in\argmin_{q\in Q(I)}\{\sum_{I'\in I_\downarrow(q)}w_{I'}+\bar\alpha(q)\},
	$
	and subsequently sets $w_I:=\sum_{I'\in I_\downarrow(\hat q)}w_{I'}+\bar\alpha(\hat q)$. 
	This procedure requires a computing time linear in $|\I_i|$. 
	Then, a maximally violated inequality can be found by building a plan according to the sequences determined at the previous step.
	%
\end{proof}

\thPolyRec*
\begin{proof}
	Due to the equivalence between optimization and separation~\cite{grotschel2012geometric}, since the separation problem for $\mathcal{D}$ can be solved in polynomial time one can solve $\mathcal{D}$ in polynomial time via the ellipsoid method~\cite{khachiyan1980}.
	As the ellipsoid method solves a primal-dual system encompassing both $\mathcal{D}$ and~\encircle{$\mathsf{A}$}, it also produces a solution to~\encircle{$\mathsf{A}$}.
\end{proof}

\corPolySol*
\begin{proof}
	First, the ellipsoid method returns an optimal solution $x^\ast$ with polynomial support size (say $|\textnormal{supp}(x^\ast)|=m$). 
	This is because the number of iterations is polynomial and, by adding a new inequality to the dual per iteration, we also add a new variable to the primal per each iteration.
	If $x^\ast$ is a basic feasible solution, its support is necessarily the smallest possible and we can halt the procedure.
	If not, $x^\ast$ belongs to the relative interior of a face of the polytope defined by \encircle{\textsf{A}}.
	Suppose to have  $\{e_jx\}_{j=1}^m=\textnormal{supp}(x)$, where $e_j$ is the canonical vector with a single 1 in position $j$.
	To obtain an optimal basic-feasible solution, we proceed as follows.
	Let $j:=1$.
	First, we restrict \encircle{\textsf{A}} to the optimal face by adding the constraint $\mathds{1}^\top x=\mathds{1}^\top x^\ast$.
	Then, we reoptimize the LP maximizing the objective function $e_j x$.
	If we do not obtain a basic-feasible solution, we iterate the procedure adding the constraint $e_j x = e_j x^*$ and letting $j := j + 1$.
	The dimension of the LP is reduced by 1 at each iteration (and the number of steps is polynomial, as we have one for each of the polynomially-many variables in $\textnormal{supp}(x^\ast)$). 
	This leads to optimizing over faces of \encircle{$\mathsf{A}$} of increasingly smaller dimension.
	When the dimension reaches 1, the corresponding solution is necessarily a basic one.
\end{proof}

\thNpHard*
\begin{proof}
	Let $|\R|=3$ and, $\forall a\in A$, $|\Theta_a|=1$. 
	Then, the problem amounts to computing an optimal CCE for a three player EFG, which is \textsf{NP}-hard since the reduction of~\cite[Th. 1.3]{vonStengel2008} directly applies.
\end{proof}

\thDNpHard*
\begin{proof}
	Consider the case in which $\R=3$, and $\mathcal{D}_B$ is adapted accordingly.
	%
	%
	Let $q=(q_1,q_2,q_3)$.
	Given dual variables $(\bar\alpha_1,\bar\alpha_2,\bar\alpha_3,\bar\beta_1,\bar\beta_2,\bar\beta_3,\bar\delta)$, the separation problem reads: 
	
	\begin{scriptsize}
		\begin{equation*}
		\min_{\theta,\sigma}\left\{\sum_{q\in\xi(\sigma)}\left[\sum_{i\in\R}U_i(\theta,q)-U_S(\theta,q)\right]+\frac{\bar\delta(\theta)}{\mu^\ast(\theta)}\right.
		\left.\hspace{.1cm}-\hspace{-.5cm}\sum_{q\in Q_1\times\xi(\sigma_2)\times\xi(\sigma_3)}
		\hspace{-.9cm}U_1(\theta,q)\bar\beta_1(q_1)
		-\hspace{-.9cm}\sum_{q\in \xi(\sigma_1)\times Q_2 \times\xi(\sigma_3)}\hspace{-.9cm}U_2(\theta,q)\bar\beta_2(q_2)
		-\hspace{-.9cm}\sum_{q\in \xi(\sigma_1)\times\xi(\sigma_2)\times Q_3}\hspace{-.9cm}U_3(\theta,q)\bar\beta_3(q_3)\right\}.
		\end{equation*}
	\end{scriptsize}
	
	
	Consider a setting with the following features: $\forall\theta\in\Theta^\ast$, $\bar\delta(\theta)=0$ (a valid assumption since $\delta\in\mathbb{R}^{|\Theta^\ast|})$; $\forall(\theta,q)\in\Theta^\ast\times\left(\bigtimes_{i\in\R}Q_i\right)$, $U_S(\theta,q)=U_1(\theta,q)$;
	$\forall (\theta,q)\in\Theta^\ast\times\left(\times_ {i\in\R}Q_i\right)$, $U_2(\theta,q)=U_3(\theta,q)=0$.
	Then, finding a maximally violated inequality corresponds to solving:
	$
	\argmax_{\theta,\sigma_2,\sigma_3}\{\sum_{q\in Q_1\times\xi(\sigma_2)\times\xi(\sigma_3)}U_1(\theta,q)\bar\beta_1(q_1)\}.
	$
	Let $U'_1\in\mathbb{R}^{|\Theta^\ast|\times|Q_2|\times|Q_3|}$ be such that, for each $(\theta,q_2,q_3)$, $U'_1(\theta,q_2,q_3)=\sum_{q_1\in Q_1}U_1(\theta,q_1,q_2,q_3)\bar\beta(q_1)$.
	If $\Theta^\ast$ is a singleton, the problem becomes 
	$
	\argmax_{\sigma_2,\sigma_3}\{\sum_{(q_2,q_3)\in\xi(\sigma_2)\times\xi(\sigma_3)}U'_1(q_1,q_2)\}.
	$
	This is a joint best-response problem between receivers 2 and 3, which is known to be \textsf{NP}-hard~\cite{vonStengel2008}.
	Therefore, the separation problem for constraints corresponding to primal variables $\gamma$ is \textsf{NP}-hard.
	Due to the equivalence between optimization and separation~\cite{grotschel2012geometric}, it follows that it is \textsf{NP}-hard to solve $\mathcal{D}_B$.
\end{proof}

\end{document}